\documentclass{article}
\pdfoutput=1

     \usepackage[preprint, nonatbib]{neurips_2019}

\usepackage[utf8]{inputenc} 
\usepackage[T1]{fontenc}    
\usepackage{hyperref}       
\usepackage{url}            
\usepackage{booktabs}       
\usepackage{amsfonts}       
\usepackage{nicefrac}       
\usepackage{microtype}      

\usepackage{amsmath}
\usepackage{subfig}
\usepackage{xspace}
\usepackage{graphicx}
\usepackage{algorithm}
\usepackage{algorithmic}
\usepackage{varwidth}
\usepackage{balance}
\usepackage{bm}
\usepackage{multirow}

\newtheorem{theorem}{Theorem}
\newtheorem{definition}{Definition}
\newtheorem{lemma}{Lemma}

\newtheorem{proof}{Proof}
\newtheorem{proofsketch}{Proof Sketch}

\newcommand{\argmax}{\mathop{\rm arg~max}\limits}

\newcommand{\Expe}[2]{\ensuremath{\mathbb{E}_{#1}\left[#2\right]}}
\newcommand{\State}{\ensuremath{\mathcal{S}}}
\newcommand{\action}{\ensuremath{\mathcal{A}}}
\newcommand{\trans}{\ensuremath{\mathcal{T}}}
\newcommand{\reward}{\ensuremath{R}}

\newcommand{\Real}{\mathbb{R}}
\newcommand{\rpos}{\mathbb{R}_{\geq 0}}
\newcommand{\his}{\mathcal{H}}

\newcommand{\method}{\textsc{PGC}\xspace}

\title{Locally Private Distributed Reinforcement Learning}

\author{%
  Hajime Ono\thanks{Equal contribution.}\\
  LINE Corporation\\
  University of Tsukuba\\
  \texttt{hajime@mdl.cs.tsukuba.ac.jp} \\
  \And
  Tsubasa Takahashi\footnotemark[1]\\
  LINE Corporation\\
  \texttt{tsubasa.takahashi@linecorp.com} \\
}

\begin{document}

\maketitle

\begin{abstract}
We study locally differentially private algorithms for reinforcement learning to obtain a robust policy that performs well across distributed private environments.
Our algorithm protects the information of local agents' models from being exploited by adversarial reverse engineering.
Since a local policy is strongly being affected by the individual environment, the output of the agent may release the private information unconsciously.
In our proposed algorithm, local agents update the model in their environments and report noisy gradients designed to satisfy local differential privacy (LDP) that gives a rigorous local privacy guarantee.
By utilizing a set of reported noisy gradients, a central aggregator updates its model and delivers it to different local agents.
In our empirical evaluation, we demonstrate how our method performs well under LDP.
To the best of our knowledge, this is the first work that actualizes distributed reinforcement learning under LDP.
This work enables us to obtain a robust agent that performs well across distributed private environments.
\end{abstract}

\section{Introduction}

Recent advancement of reinforcement learning (RL) shows great success within broad domains ranging from market strategy decisions~\cite{Abe2004marketing}, load balancing~\cite{Cogill2006load} to autonomous driving~\cite{Shalev-Shwartz2016driving}.
Reinforcement learning is a process to obtain a good policy in a given environment in an unsupervised learning manner.
Distributed reinforcement learning (DRL) is known as a practical solution to accelerate reinforcement learning in parallel~\cite{Mnih2016a3c,nair2015golira, Palmer2019multi, Bacchiani2019multi}.
It also gives us \textit{robust} policies across different environments.
A policy is regarded as \textit{robust} if it performs well across various environments, but not overfitting a simulated environment.
A policy overfitting a simulated environment does not work well for the real-world environment~\cite{rajeswaran2016epopt}.

In case that the local environments are related to private information, such as private rooms and individual properties, there are privacy issues.
Since a locally learned policy is strongly being affected by the individual environment, the output of the agent may release the private information unconsciously.
\cite{pan2019you} pointed out that reinforcement learning can cause privacy issue.
They proposed an attack to recover the dynamics of agents through estimating the transition dynamics with their state space, action space, reward function, and trained policy.
For example, from the policy of a robot cleaner trained in an individual's room, an adversary can estimate the room layout if she can access the policy.
In the distributed settings, sending information by the local agent has serious privacy risks if we do not believe the central aggregator.

Local differential privacy (LDP) \cite{kasiviswanathan2011can, Duchi2013minimax} gives a rigorous privacy guarantee when data providers send information to a data curator.
Mechanisms ensuring LDP makes outputs indistinguishable values regardless of the input.
In this paper, we aim to design locally differentially private algorithms for DRL, such that reported information from local agents is indistinguishable.

\begin{figure}[t]
    \centering
    \includegraphics[width=0.68\hsize]{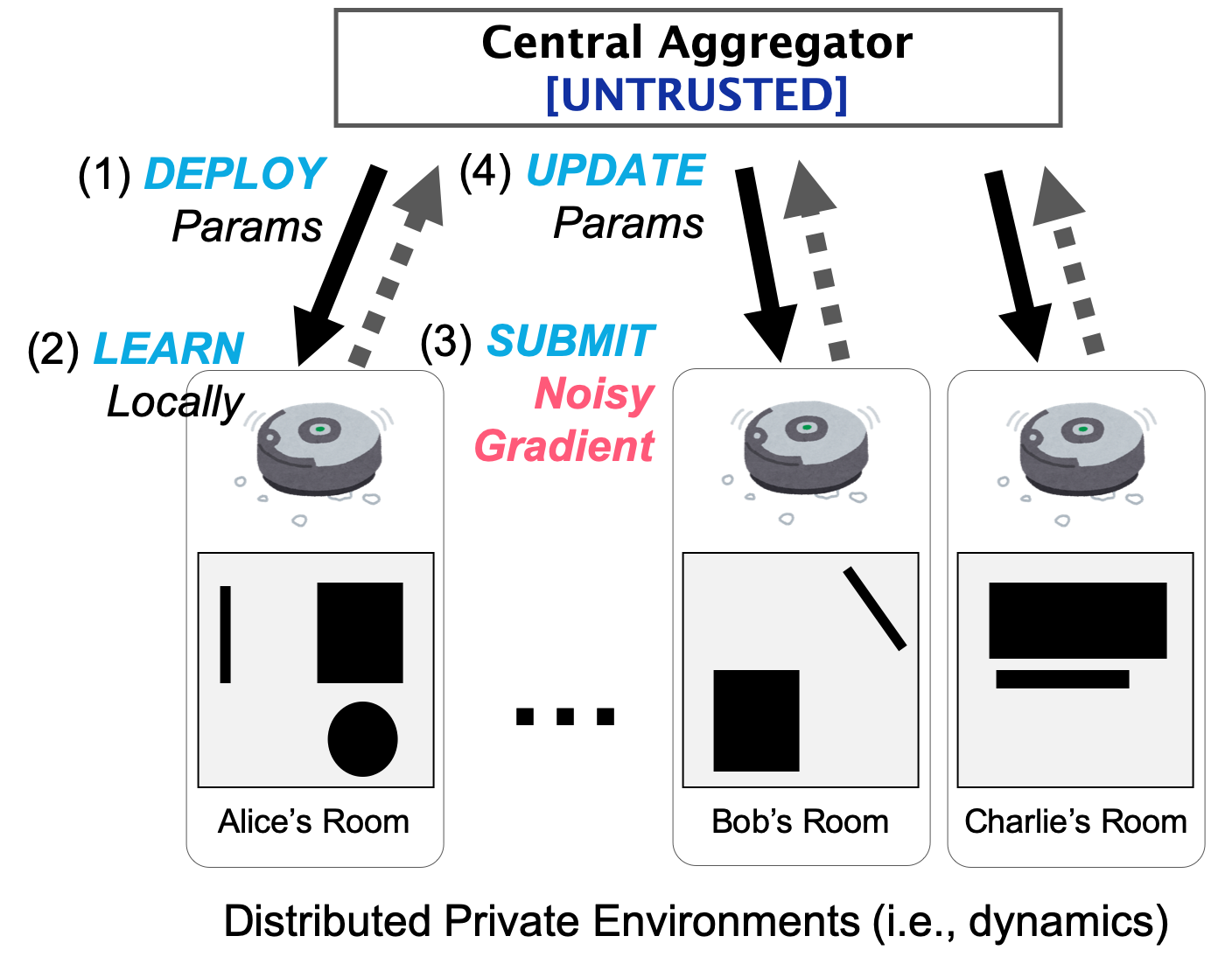}
    \caption{Private Gradient Collection (\method) framework. The framework aims to lean a robust policy based on the reported noisy gradients that satisfy LDP from local agents. The central aggregator updates his model by the noisy gradients, and distributes the updated model to different agents for making the policy robust across various environments.}
    \label{fig:overview}
\end{figure}

To achieve the DRL under LDP constraints, we develop a framework that leans a robust policy based on the reported information from the agents while preserving local privacy of them (Figure \ref{fig:overview}).
We call the framework \textit{Private Gradient Collection} (\method).
In the framework, first, the central aggregator distributes a global model to several local agents.
Second, the local agents update the model at local private environments.
Third, the agents report noisy gradients that satisfy LDP to the central aggregator.
At last, the central aggregator updates the global parameters by utilizing a set of reported noisy gradients.
After updating the global model, the central aggregator distributes the model to the other agents to learn more and more. 
Following the above way, local agents can report their updates by submitting noisy gradients even if the local nodes do not have any deliverable data.
Besides, the central aggregator easily updates the global model by just applying the collected gradients to the model, without any privacy concerns of the local agents.

To the concrete realization of the framework, we introduce an algorithm based on asynchronous advantage actor-critic (A3C) based.
For introducing randomness that satisfies LDP, we present two mechanisms for the gradient submission.

To the best of our knowledge, this is the first work that actualizes  DRL under local differential privacy.
In this paper, we show that our algorithm ensures LDP guarantees by utilizing a series of techniques.
In our empirical evaluations, we demonstrate how our method learns the robust policy effectively even it is required to satisfy local differential privacy.
This work enables us to obtain a robust agent that performs well across distributed private environments.

\subsection{Related Works}

For privacy-preserving distributed reinforcement learning, both cryptographical and differentially private approaches have been studied.

The cryptographical approaches conceal information during their learning process \cite{Zhang2005nego, Sakuma2008pprl}.
However, if several agents cooperate, they have chances to estimate the information of the other agents.
Our proposed method under LDP is robust against the cooperated adversarial parties.
Even if all other remaining agents attack an agent, the agent's dynamics are indistinguishable from the different candidate dynamics.

Noisy DQN~\cite{Fortunato2018noisy} is a DQN~\cite{Mnih2015human} variant that aims to improve learning stability by injecting Gaussian noise, but it has no way to preserve privacy.
\cite{wang2019privacy} introduced differentially private Q-learning in continuous spaces.
\cite{zhu2019applying} introduced a cooperative multi-agent system that chooses advice information from neighboring agents in a differentially private way.
\cite{chamikara2019local} proposed a private distributed learning framework that craft perturbed data satisfying LDP at local data holders and send it to the untrusted curator.
We focus on DRL that agents do not have such deliverable data, but submit perturbed gradients.
\section{Preliminaries}

Before detail discussions, we introduce essential background notations, definitions, and related works to understand our proposals.

\subsection{(Local) Differential Privacy}\label{ssec:ldp}

Differential privacy \cite{dwork2006cali} is a rigorous privacy definition, which quantitatively evaluates the degree of privacy protection when releasing statistical aggregates.
While \textit{local differential privacy} (LDP) \cite{kasiviswanathan2011can, Duchi2013minimax} gives a rigorous privacy guarantee when data providers send information to a data curator.
Suppose the data providers send information to a collector via some random mechanism $\mathcal{Q}$.
\begin{definition}
For all possible input $x, x' \in X$ and for any subset of outputs $Y \subseteq Z$,
a randomized mechanism $Q$ satisfies $\varepsilon$-local differential privacy if it holds that
\begin{equation}
  \frac{\Pr(\mathcal{Q}(x)\in Y)}{\Pr(\mathcal{Q}(x')\in Y)}\leq e^\varepsilon 
  \label{eq:ldp}
\end{equation}
where $e$ is the Napier number.
\end{definition}
The definition requires $\mathcal{Q}$ to output indistinguishable values regardless of the input.
An essential property of a mechanism $\mathcal{Q}$ is the (global) sensitivity of the output.

\begin{definition}
For all input $x, x'\in X$, the sensitivity of a mechanism $Q$ is defined as
\begin{equation}
    \Delta_{\mathcal{Q}} = \sup_{x, x'\in X}||\mathcal{Q}(x) - \mathcal{Q}(x')||.
    \label{eq:sensitivity}
\end{equation}
\end{definition}

The sequential composition is the property that describes an intuition that more outputs more violate privacy.
\begin{theorem}\label{thm:seqcomp}
Let $\mathcal{Q}_1,\mathcal{Q}_2,\hdots,\mathcal{Q}_K$ be a series of mechanisms.
Assume that $\mathcal{Q}_i$ satisfies $\varepsilon_i$-(L)DP for each $i=1,\hdots, K$, respectively.
Then, the series of mechanisms satisfies $(\sum_{i\in[K]}\varepsilon_i)$-(L)DP.
\end{theorem}

Post-processing invariance is a property that differentially private information never harm privacy anymore.
\begin{theorem}
For any deterministic or randomized function $f$ defined over the mechanism $\mathcal{Q}$, if $\mathcal{Q}$ satisfies $\varepsilon$-(L)DP, $f(\mathcal{Q}(x))$ also satisfies $\varepsilon$-(L)DP for any input $x \in X$.
\end{theorem}
Because of the property, in a local private setting, the curator is allowed to run arbitrary processing for the collected data.

\textbf{Laplace mechanism.}
Laplace mechanism is the well known randomized mechanism that samples randomized values from the Laplace distribution.
The Laplace distribution is designed based on the sensitivity of the target function outputs. 
The mechanism samples the randomized output $\tilde{y}$ from the Laplace distributions denoted as
\begin{equation}
  \mathbb{P}(\tilde{y}_i|y_i) =\frac{\varepsilon}{2\Delta_{\mathcal{Q}}}\exp\left(-\frac{\varepsilon}{\Delta_{\mathcal{Q}}}|\tilde{y}_i-y_i|\right) 
\end{equation}
where $y_i$ is the $i$-th element of $y$.

\textbf{Bit flip.}
Bit flip~\cite{Ding2018comparing,Ding2017telemetry} $\mathcal{Q}_{\text{bf}}:[-1, 1]\rightarrow\{-1, +1\}$ is a randomization technique for satisfying (L)DP.
For input $x\in[-1, +1]$,
\begin{equation}
  \mathcal{Q}_{\text{bf}}(x) =
  \begin{cases}
    +1 & \text{with probability } \frac{1}{e^\varepsilon+1} + \frac{x+1}{2}\frac{e^\varepsilon-1}{e^\varepsilon+1}\\
    -1 & \text{otherwise}.
  \end{cases}  
  \label{eq:bf}
\end{equation}
By the bit flip, the randomized outputs have sharp directions.

Random projection ~\cite{Johnson1984jl, Achlioptas2001random, Bingham2001random} is a useful technique that reduces the dimensionality of the vector by a random matrix.
We can use the random matrix $M=(m_{i,j})\in\Real^{\hat{d}\times d}$ such that
\begin{equation}
  \forall i\in[\hat{d}], j\in[d], m_{i,j} = \begin{cases}
    -\sqrt{3} & \text{w/ probability } \frac{1}{6}\\
    0 & \text{w/ probability } \frac{2}{3}\\
    +\sqrt{3} & \text{w/ probability } \frac{1}{6},
  \end{cases}
  \label{eq:randomMatrix}
\end{equation}
where $\hat{d}$ is the dimension of mapped space.
The random matrix has the useful property that the column vectors of $M$ are almost orthogonal each other~\cite{Achlioptas2001random, Bingham2001random}.
Thanks to the property, we can approximately recover the original vector using the transposed matrix $M$ from the compressed vector.

\subsection{Distributed Reinforcement Learning}

Benefits of distributed reinforcement learning (DRL) are increasing learning efficiency and obtaining a robust policy.
We focus on the later while preserving the privacy of local agents.
Not only the distributed setting, but robust RL is also the generalization of RL, which adapts uncertainty of the transition dynamics~\cite{Morimoto2005robust,rajeswaran2016epopt, Pinto2017robust}.
Our goal is to obtain a transferable policy that performs well across various environments.

Suppose there are a central aggregator and $N$ distributed agents.
Agent $n$ moves around on the Markov decision process~(MDP) which is characterized with common state space $\State$, common action space $\action$, common reward function $\reward:\State\times\action\rightarrow\rpos$, common discounting factor $\gamma\in(0, 1)$ and local transition dynamics $\trans(\psi_n)$.
$\State$ contains some terminal states.
Each local dynamics $\trans(\psi_n):\State\times\action\rightarrow\Delta_{\State}$ decides the next state after an action on a state where $\Delta_{\State}$ is the probability simplex on \State. 
Local dynamics $\trans(\psi_n)$ is parametrized with $\psi_n\in\Psi$ where $\Psi$ is a parameter set.
For each round $t$, agent $n$ has state $s_{n,t}\in\State$ and takes action $a_{n,t}\in\action$.
After the action, the agent gets reward $r_{n,t}$ from $\reward$, and the state transits $s_{t+1}$.
$\reward$ gives $0$ to the terminal states.
The transition follows local transition dynamics $\trans(\psi_n)$.
Agent $n$ decides its action by policy $\pi:\State\rightarrow\Delta_{\action}$, which is shared by all agents.
The central aggregator trains the policy with the cooperation of the agents.
Defining $\rho$ as the initial state distribution, for each agent $n\in[N]=\{1,2,\hdots,N\}$, 
history $\his_n=(s_{n,0},a_{n,0},s_{n,1},a_{n,1},\hdots)$ is the random variable such that
\begin{equation}
\begin{split}
    \mathbb{P}(\his_n|\pi) = \rho(s_{n,0})\prod_{t=1}^{\infty} & \mathbb{P}(a_{n, t-1}|\pi, s_{n, t-1}) \\
    &\mathbb{P}(s_{n,t}|\trans(\psi_n), s_{n,t-1}, a_{n, t-1}).  
\end{split}
\end{equation}
where $\{\his_n\}_{n\in[N]}$ is determined with the MDP and the policy.
To obtain a robust policy, which works well on some dynamics in the possible dynamics set $\{\trans(\psi):\psi\in\Psi\}$, we solve an optimization problem.
The objective function is
\begin{equation}
  \argmax_{\pi}\sum_{n\in[N]}\Expe{\his_n}{\sum_{t=0}^{\infty}\gamma^t\reward(s_{n, t},a_{n,t})}.
  \label{eq:obj}
\end{equation}
where $\gamma\in(0,1)$ is the discounting factor. 

\subsubsection{Asynchronous Advantage Actor-critic}

\textit{Asynchronous advantage actor-critic}~(\textit{A3C})~\cite{Mnih2016a3c} is a DRL framework, which is originally proposed for acceleration of policy training in parallel.
On distributed A3C protocol, each agent optimizes both policy and approximation of a state-value function.
The policy is denoted as $\pi(\cdot|\cdot;\theta^{(c)},\theta^{(p)}):\action\times\State\rightarrow[0, 1]$.
For some $n\in[N]$ and some $t$, $\pi(a_{n,t}=a|s_{n,t}=s;\theta^{(c)},\theta^{(p)})$ represents the confidence for action $a$ on $s$.
Based on the confidence, each agent decides the next action at each time step.
State-value function $V^\pi:\State\rightarrow\rpos$ is the function which represents the value of each state on policy $\pi$.
\begin{equation}
  V^{\pi}(s) \equiv \Expe{}{\sum_{k=0}^\infty\gamma^kr_{t+k}\Bigg|s_t=s}.
\end{equation}
Further, $V(\cdot;\theta^{(c)},\theta^{(v)}): \State\rightarrow\Real$ denotes the approximated state value function.
For simplicity, this paper denotes $\pi(\cdot|\cdot;\theta)$ and $V(\cdot;\theta)$ as $\pi(\cdot|\cdot;\theta^{(c)},\theta^{(p)})$ and $V(\cdot;\theta^{(c)},\theta^{(v)})$.

The objective function is defined as follows:
\begin{equation}
    L(\theta) = \sum_{n\in[N]}L^{(p)}_n(\theta) + \lambda L^{(v)}_n(\theta)
    \label{eq:a3c_obj}
\end{equation}
where
\begin{equation}
    \begin{split}
        L^{(p)}_n = -\mathbb{E}_{}\Bigg[&\sum_{t=0}^{T_n}\log(\pi(a_{n,t}|s_{n,t};\theta))A(s_{n,t};\theta)\\
    &-\beta \sum_{t=0}^{T_n}H(\pi(\cdot|s_{n,t};\theta))\Bigg],\\
    \end{split}
\end{equation}
\begin{equation}
    L^{(v)}_n = \Expe{}{\sum_{t=0}^{T_n}(V^\pi(s_{n,t})-V(s_{n,t};\theta))^2}
\end{equation}
\begin{equation}
    \begin{split}
        A(s_{n,t};\theta) = &\sum_{j=t}^{T_n-1}\left(\gamma^{j-t}r_{j}\right) \\
    &+\gamma^{T_n-t}V(s_{n,T_n};\theta) - V(s_{n,t};\theta)
    \end{split}
\end{equation}
\begin{equation}
    H(\pi(\cdot|s_{n,t};\theta)) = \sum_{a\in\action}-\pi(a|s_{n,t};\theta)\log\left(\pi(a|s_{n,t};\theta)\right)
\end{equation}
with terminal state $s_{n, T_n}$ and some small positive real values $\beta$ and $\lambda$.
$L^{(p)}(\theta)$ is the loss of policy, and $L^{(v)}(\theta)$ is the loss of the estimation of value function.
Decreasing (\ref{eq:a3c_obj}) increases the true objective (\ref{eq:obj}).

\section{Locally Differentially Private Actor-Critic}

\subsection{Our Algorithm}

We here present our algorithm for DRL under local differential privacy.
We first introduce an overview and an abstract model of our algorithm.
We call the abstract model \textit{Private Gradient Collection} (\method).
Based on the model, we present our algorithm \method-A3C that is a method based on A3C with satisfying $\varepsilon$-LDP for all local agents.
We also address the privacy analysis of the proposed method and give some extensions.

\subsubsection{Private Gradient Collection}
Suppose the central aggregator has a model parameterized with $\theta\in\Real^d$  where $d$ is the dimensionality of $\theta$, and trains $\theta$ by utilizing reported information from local agents.
The central aggregator and all local agents share the parameters $\theta$, the structure of the model, and loss function $L$.

The abstract model \method follows the below four steps:
\begin{enumerate}
\setlength{\parskip}{0pt}
\setlength{\itemsep}{0pt}
\setlength{\leftskip}{0pt}
    \item The central aggregator delivers $\theta$ to local agents.
    \item Each local agent initializes her parameters $\theta'$ by $\theta$ and updates $\theta'$ in her local private environment.
    \item The local agent reports information about her model with injecting noise to satisfy $\varepsilon$-LDP.
    \item The central aggregator updates $\theta$ by only utilizing the received noisy information from the local agents.
\end{enumerate}

The primal question to design the model is what information should local agents report?
Our answer is \textit{stochastic gradient}.
Hence, the local agent computes loss and its gradient behind local observations and rewards, then submits a \textit{noisy gradient} to the central node.

In the local training process, each agent inputs $s_{n,t}$ to the network and obtains the next action $a_{n,t}$ or some information to decide the action.
At the end of an episode, with history $\his_n$ and observed rewards $\{r_{n,1},r_{n,2},\hdots,\}$, agent $n$ evaluates $\theta'$ by the loss function L.
After the evaluation, the agent computes stochastic gradient $g_n$ of $\ell$ along $\theta'$.
Before reporting to the central aggregator, she randomizes the gradient. 
The randomness (e.g., additive noise) is designed to satisfy $\varepsilon$-LDP via random mechanism $\mathcal{Q}$. 
\begin{definition}($\varepsilon$-LDP for gradient submissions)
For each $n\in[N]$, any $\psi,\psi'\in\Psi$ and any subset $G\subset \Real^d$, the following inequality must hold:
    \begin{equation}
        \frac{\Pr(\tilde{\mathbf{g}}_n\in G|\psi_n=\psi, \tilde{\mathbf{g}}_1,\hdots,\tilde{\mathbf{g}}_{n-1})}
    {\Pr(\tilde{\mathbf{g}}_n\in G|\psi_n=\psi',  \tilde{\mathbf{g}}_1,\hdots,\tilde{\mathbf{g}}_{n-1})}\leq e^\varepsilon.
    \end{equation}
\end{definition}

With the noisy gradients, the central aggregator updates $\theta$.
Then, the updated parameters are shared with all distributed agents again.
To make the problem simple, we assume that one agent submits the gradient only once.
Because of the post-processing invariant, the central aggregator can apply the noisy gradients to the parameters in any way.
She can use any gradient method and can use a submitted gradient multi-time.

\subsubsection{\method-A3C}

As a concrete realization of the \method framework, we propose \method-A3C, which is an LDP variant of A3C.
Following the \method framework, \method-A3C employs the gradient submissions with a randomized mechanism $\mathcal{Q}$ from local agents to the central aggregator.
The other procedure follows the original algorithm of A3C.
Algorithm \label{alg:ldpa3c} is the overall procedure of \method-A3C.

\paragraph{Empirical Loss Minimization.}
As well as vanilla A3C, based on the episode history $\his_n$, the local agent evaluates empirical loss:
\begin{equation}
    \hat{L}_n(\theta';\his_n, \lambda) \equiv \hat{L}^{(p)}_n(\theta';\his_n) + \lambda \hat{L}^{(v)}_n(\theta';\his_n)
\end{equation}
where
\begin{equation}
    \begin{split}
        \hat{L}^{(p)}_n(\theta';\his_n) =& -\sum_{t=0}^{T_n}\log(\pi(a_{n,t}|s_{n,t};\theta'))A(s_{n,t};\theta')\\
    &-\beta \sum_{t=0}^{T_n}H(\pi(\cdot|s_{n,t};\theta')) \\
    \end{split}
\end{equation}
\begin{equation}
    \hat{L}^{(v)}_n(\theta';\his_n) = \sum_{t=0}^{T_n}(\hat{V}(s_{n,t})-V(s_{n,t};\theta'))^2
\end{equation}
\begin{equation}
    \hat{V}(s_{n,t}) = \sum_{j=t}^{T_n-1}\gamma^{j-t}r_{j}+V(s_{n,T_n};\theta').
\end{equation}
The empirical loss replaces random value $\his_n$ in Equation \ref{eq:a3c_obj} with observed $\his_n$.
After the evaluation, the agent computes stochastic gradient $\mathbf{g}_n$ of the empirical loss along $\theta'$.
For the stochastic gradient, each agent crafts the noisy gradient $\hat{\mathbf{g}}_n$ by a randomized mechanism $\mathcal{Q}$ to satisfy $\varepsilon$-LDP, and submits $\hat{\mathbf{g}}_n$ to the central aggregator.

\paragraph{Crafting Noisy Gradient.}
We discuss how to craft a noisy gradient that satisfies $\varepsilon$-LDP.
A simplest way to craft the gradient is to follow the Laplace mechanism that is well-known in differential privacy literature.
However, for a stochastic gradient, it is hard to deal with its sensitivity.
We employ clipping technique to bound the sensitivity.
\begin{equation}
\bar{\mathbf{g}} = \frac{\mathbf{g}}{\max\{1,\frac{\|\mathbf{g}\|_1}{C/2}\}}
\label{eq:clip}
\end{equation}
where $\mathbf{g}$ is a stochastic gradient vector, $C$ is clipping size.
Each agent clips the gradient by the norm with a positive constant $C/2$, and then the sensitivity is bounded by $C$.
That is, any two clipped gradients $\bar{\mathbf{g}},\bar{\mathbf{g}}'$ satisfies
\begin{align}
    \|\bar{\mathbf{g}}-\bar{\mathbf{g}}'\|_1 \leq C.
    \label{eq:bound_sen}
\end{align}
Based on the clipping (\ref{eq:clip}) and the sensitivity bounded by $C$ (\ref{eq:bound_sen}), each agent generate the Laplace noise $\mathbf{z}$ such that:
\begin{equation}
    \mathbb{P}(z_i) =\frac{\varepsilon}{2C}\exp\left(-\frac{\varepsilon}{C}|z_i|\right)
    \label{eq:lapnoise}
\end{equation}
where $z_i$ is the $i$-th dimensional value of $\mathbf{z}$.
With the noise $\mathbf{z}$, each agent report noisy gradient $\bar{\mathbf{g}}+\mathbf{z}$ to the central aggregator.
This procedure is described in Algorithm \ref{alg:laplace}.

\paragraph{Updating Global Parameter with Buffer.}
The central aggregator updates his global parameter $\theta$ by received gradients from the local agents.
To reduce the variance of the noisy gradients, we introduce a temporal storage \textit{buffer}.
The central aggregator first stores multiple noisy gradients into the buffer $B$, and update $\theta$ with utilizing all $\tilde{\mathbf{g}} \in B$ as
\begin{equation}
    \theta = \theta - \eta \frac{1}{|B|}\sum_{\tilde{\mathbf{g}}\in B}\tilde{\mathbf{g}}.
    \label{eq:update_param}
\end{equation}
The central aggregator does not utilize any other information about local agents except received noisy gradients.
Therefore, the update process, which is the post-processing of all received gradients, does not violate $\varepsilon$-LDP for any local agents.
After updating $\theta$, the central aggregator flushes the buffer $B$.
We expect that the buffering improves learning stability as well as mini-batch learning.

\begin{algorithm}[tb]
    \caption{\method-A3C}
    \label{alg1}
    \begin{algorithmic}
        \STATE {\bfseries Input:} $N$ agents, reward function $R$, and randomized mechanism $\mathcal{Q}$
        \STATE {\bfseries Parameters:} privacy parameter $\varepsilon$, reduced dimension $\hat{d}$, clip parameter $C$, maximum buffer size $\text{MAX\_BUF}$, learning rate $0<\eta<1$ and scale parameter $\lambda$
        
        \STATE{Initialize buffer $B \leftarrow\{\}$}
        \STATE{Initialize global parameters $\theta$}

        \REPEAT
            \STATE{\textbf{// agent $n$ asynchronously begins a local process}}
            \STATE{Copy parameters $\theta'\leftarrow\theta$ \label{line:copy}}
            \STATE{Initialize step counter $t\leftarrow 0$}
            \STATE{Get initial state $s_{n,0}$ \label{line:initial_state}}
            
            \WHILE{$s_{n,t}$ is not a terminal state}
                \STATE{$a_{n,t} = \argmax_{a\in\action} \pi(a|s_{n,t}; \theta')$}
                \STATE{Receive $r_{n,t}$ and $s_{n,t+1}$ following $\reward$ and $\trans(\psi_n)$}
                \STATE{$t\leftarrow t+1$}
            \ENDWHILE
            
            \STATE{$\mathbf{g}_n\leftarrow \frac{\partial}{\partial \theta'}\hat{L}_n(\theta';\his_n, \lambda)$\label{line:grad}}
            \STATE{$\tilde{\mathbf{g}}_n \leftarrow \mathcal{Q}(\mathbf{g}_n;\varepsilon, d, \hat{d}, C)$\label{line:randomize}}
            \STATE{Send $\tilde{\mathbf{g}}_n$ to central controller\label{line:submit}}
            
            \STATE{\textbf{// agent $n$ ends the local process}}
            \STATE{$B \leftarrow B\cup \{\tilde{\mathbf{g}}_n\}$}
            \IF{$|B|\geq \text{MAX\_BUF}$}
                \STATE{Perform asynchronous update of $\theta$ with $B$ by (\ref{eq:update_param})}
                \STATE{$B\leftarrow\{\}$ // \textbf{clear buffer}}
            \ENDIF
        \UNTIL{$N$ submissions received}
    \end{algorithmic}
\end{algorithm}

\begin{algorithm}[t]
    \caption{Laplace mechanism (for gradient submissions)}
    \label{alg:laplace}
    \begin{algorithmic}
        \STATE {\bfseries Input:} gradient vector $\mathbf{g}$, privacy parameter $\varepsilon$, dimensionality of the gradient vector $d$,  and clipping size $C$
        \STATE{$\bar{\mathbf{g}}\leftarrow$ clip vector $\mathbf{g}$ as (\ref{eq:clip})}
        \FOR{$i\in[d]$}
            \STATE {generate $z_i$ such that (\ref{eq:lapnoise}).}
        \ENDFOR
        \STATE {\bfseries Output:} $\bar{\mathbf{g}} + \mathbf{z}$
    \end{algorithmic}
\end{algorithm}

\begin{algorithm}[t]
    \caption{Projected Random Sign Mechanism}
    \label{alg:bit_flip}
    \begin{algorithmic}
        \STATE{\bfseries Input:} gradient vector $\mathbf{g}$, privacy parameter $\varepsilon$, original dimension $d$, reduced dimension $\hat{d}$ and clipping size $C$
        \STATE{Generate random matrix $\mathbf{M}\in\Real^{\hat{d}\times d}$ as (\ref{eq:randomMatrix})}
        \STATE{$\mathbf{u}\leftarrow \mathbf{M} \mathbf{g}$}
        \FOR{$i\in[\hat{d}]$}
            \STATE{
                $\bar{u}_i \leftarrow$ clip $u_i$ as (\ref{eq:clip_prs})
            }
            \STATE{
                $\tilde{u}_i \leftarrow$ randomize $\bar{u}_i$ as (\ref{eq:bitflip})
            }
        \ENDFOR
        \STATE{\bfseries Output:} $\mathbf{M}^\top \tilde{\mathbf{u}}$
    \end{algorithmic}
\end{algorithm}

\subsubsection{Acceleration of Learning Efficiency}

Since the Laplace mechanism is simple, but decreases the accuracy of gradient significantly, the learning efficiency of the whole DRL might be decreased.
We introduce an alternative randomizing technique, projected random sign (PRS), to have opportunity to increase the learning efficiency. 
To increase the learning effieicy, the PRS mechanism addresses reducing dimensionality and sharpning gradient direction while injecting randomness for LDP.

First, the PRS applies the random projection to reduce dimensionality.
Each agent maps $d$ dimensional stochastic gradient vector $\mathbf{g}$ to $\hat{d} (< d)$ dimensional vector $\mathbf{u}$ with random matrix $\mathbf{M}$ as $\mathbf{u}=\mathbf{M}\mathbf{g}$.
$\mathbf{M}$ follows (\ref{eq:randomMatrix}).

Second, before applying randomization, each agent applies \textit{element-wise clipping} denoted as follows:
\begin{equation}
    \bar{u}_i = 
        \begin{cases}
            -C & \text{if } u_i < -C,\\
            +C & \text{if } u_i > +C,\\
            u_i & \text{otherwise}
        \end{cases}
    \label{eq:clip_prs}
\end{equation}
where $u_i$ is the $i$-th dimensional value of $\mathbf{u}$.
This element-wise clipping bounds the sensitivity by $C$.

Third, we apply bit flipping for the clipped vector produced by (\ref{eq:clip_prs}).
The bit flipping extending (\ref{eq:bf}) is denote as follows:
\begin{equation}
    \tilde{u}_i =
        \begin{cases}
            +C & \text{w/ probability } \frac{1}{e^{\varepsilon/k}+1} + \frac{\bar{\mathbf{u}}+C}{2C}\frac{e^{\varepsilon/\hat{d}}-1}{e^{\varepsilon/\hat{d}}+1},\\
            -C & \text{otherwise}.
        \end{cases}
    \label{eq:bitflip}
\end{equation}
The randomization by (\ref{eq:bitflip}) consumes privacy parameter $\varepsilon/\hat{d}$ for each dimension.
At last, this mechanism inversely transforms $\hat{\mathbf{u}}$ as $\mathbf{M}^{\top}\hat{\mathbf{u}}$, and report it to the central aggregator.

\subsection{Privacy Analysis}

\begin{lemma}\label{thm:prima}
    An algorithm that follows the \method framework satisfies $\varepsilon$-LDP for all local agents.
\end{lemma}
\begin{proofsketch}
    In step 3 of the \method framework, each agent reports a noisy gradient that ensures $\varepsilon$-LDP.
    In step 4, the central aggregator updates $\theta$ only utilizing the received noisy gradients from the local agents.
    This step is independent of any information about local agents.
    Therefore, steps 4 does not violate $\varepsilon$-LDP due to post-processing invariance.
    Move forward to the next round. The central aggregator delivers the updated parameter $\theta$ to the other agent at step 1.
    At step 2, different agent copies the parameter as $\theta'=\theta$ and updates $\theta'$ through her local environment.
    Since the learning process at a local agent is independent of all other agents, the output $\theta'$ also does not violate $\varepsilon$-LDP for all other agents.
\end{proofsketch}

\begin{lemma}\label{thm:laplace}
    Gradient submission with the Laplace mechanism satisfies $\varepsilon$-LDP.
\end{lemma}
\begin{proof}
    Each agent is given $\theta'$ which contains the information of $\{\tilde{\mathbf{g}}_{n^*}\}_{n^*\in \mathcal{N}}$ where $\mathcal{N}\subset [n-1]$.
    With $\theta'$, agent $n$ computes gradient $\mathbf{g}_n$ and outputs $\tilde{\mathbf{g}}_n$.
    With the clipping and Laplace mechanism, for any $\mathbf{g}_n$ and $\mathbf{g}_n'$, the following inequality holds.
    \begin{align*}
        \frac{\mathbb{P}(\tilde{\mathbf{g}}_n|\mathbf{g}_n)}{\mathbb{P}(\tilde{\mathbf{g}}_n|\mathbf{g}_n')}\leq e^\varepsilon.
    \end{align*}
    Since the inequality holds regardless $\theta'$,  the following also holds.
    For any $\psi_n, \psi_n'\in\Psi$,
    \begin{align*}
        \frac{\mathbb{P}(\tilde{\mathbf{g}}_n|\psi_n, \tilde{\mathbf{g}}_1,\hdots,\tilde{\mathbf{g}}_{n-1})}{\mathbb{P}(\tilde{\mathbf{g}}_n|\psi_n', \tilde{\mathbf{g}}_1,\hdots,\tilde{\mathbf{g}}_{n-1})}\leq e^\varepsilon.
    \end{align*}
\end{proof}

\begin{lemma}\label{thm:prs}
    Gradient submission with the PRS mechanism satisfies $\varepsilon$-LDP.
\end{lemma}
\begin{proof}
    For any $\mathbf{g}',\mathbf{g}''\in\Real^d$, any $\mathcal{Y}\subset \Real^d$ and any $\mathcal{Y}'\subset \{-C, C\}^{\hat{d}}$ such that $\mathcal{Y}=\{\mathbf{M}^\top \mathbf{y}'\in\Real^d | \mathbf{y}'\in\mathcal{Y}'\}$, 
    \begin{align*}
        \frac{\Pr(\mathbf{M}^\top \tilde{\mathbf{u}}\in\mathcal{Y}|\mathbf{g}=\mathbf{g}')}{\Pr(\mathbf{M}^\top \tilde{\mathbf{u}}\in\mathcal{Y}|\mathbf{g}=\mathbf{g}'')}
        \leq & \frac{\Pr(\tilde{\mathbf{u}}\in\mathcal{Y}'|\mathbf{g}=\mathbf{g}')}{\Pr(\tilde{\mathbf{u}}\in\mathcal{Y}'|\mathbf{g}=\mathbf{g}'')}\\
        \leq & \max_{\mathbf{y}'\in\mathcal{Y}'}\prod_{i\in[\hat{d}]}
        \frac{\Pr(\tilde{u}_i=y_i'|\mathbf{g}=\mathbf{g}')}{\Pr(\tilde{u}_i=y_i'|\mathbf{g}=\mathbf{g}'')}\\
        \leq & \prod_{i\in[\hat{d}]}e^{\varepsilon/\hat{d}} = e^{\varepsilon}.
    \end{align*}
    The following expansion is as well as Proof of Lemma \ref{thm:laplace}.
\end{proof}

\begin{lemma}\label{thm:update}
    Updating paramters on the central aggregator (\ref{eq:update_param}) does not violate $\varepsilon$-LDP that has been satisfied for each local agent.
\end{lemma}
\begin{proof}
    The parameter update (\ref{eq:update_param}) is only utilizing received noisy gradient in the batch $B$, which means independent from any information about any local agents except noisy gradients.
    Due to post-processing invariance, the parameter update does not violate $\varepsilon$-LDP that has been satisfied at each local agent.
\end{proof}

\begin{theorem}
   \method-A3C (Algorithm \ref{alg1}) satisfies $\varepsilon$-LDP for all local agents.
\end{theorem}
\begin{proof}
    From Lemma \ref{thm:prima}, \ref{thm:laplace}, \ref{thm:prs} and \ref{thm:update}, \method-A3C obviously satisfies $\varepsilon$-LDP for all local agents.
\end{proof}

\subsubsection{Extending to multiple submissions}
We can easily extend the algorithm to a locally multi-round algorithm.
In the multi-round algorithm, each agent submits a randomized stochastic gradients $T$-times.
For each submission, the agent consumes a privacy budget $\varepsilon/T$, and the whole consumed budget is $\varepsilon$ because of the sequential composition theorem (Theorem \ref{thm:seqcomp}).

\section{Experiments}

We here demonstrate the effectiveness of our proposals.
We evaluate \textbf{learning efficiency}, \textbf{success ratio}, and  \textbf{trade-off between privacy and efficiency}.
Before showing empirical results, we describe what the evaluation task is and how to implement \method-A3C.

\textbf{Evaluation Task.} 
We make some numerical observations on \textit{cart pole} with different gravity acceleration coefficients $\psi_n \in \{9.7, 9.8, 9.9\}$.
Suppose that each coefficient is appeared uniformly.
Cart Pole~\cite{Barto1983cart} is the classical reinforcement learning task that an agent controls a cart with a pole to keep the pole standing.
The number of time steps in which the pole is standing is the cumulative reward.
The cumulative reward is called \textit{score} in this section, and the maximum score is 200.
$\State$ consists of cart position $\in[-4.8, 4.8]$, cart velocity $\in[-\infty,  \infty]$, pole angle $\in[-24 \text{deg}, 24 \text{deg}]$ and pole velocity at tip $\in[-\infty, \infty]$.
The cart moves on a one-dimensional line. 
At each time step, the agent selects an action from $\action = \{\text{push left}, \text{push right}\}$.

\textbf{Stopping Criteria.}
We hire a way to iterate the learning process until the central aggregator received the predefined number of submissions from agents.
Since we assume that each agent submits only once, the number of submissions is identical to the number of agents.
We assume the scores are not private information.
If we need to protect the score with LDP, we can easily develop additional  submissions that agents send a noisy boolean representing whether the score is larger than a threshold.

\textbf{Implementation.}
To implement the proposed algorithms, we use two shallow neural networks corresponding to $\pi(\cdot|\cdot;\theta^{(c)},\theta^{(p)})$ and $V(\cdot;\theta^{(c)},\theta^{(v)})$, respectively.
Each network has two layers, and the activation functions are ReLU.
$\theta^{(c)} \in \Real^{16\times 4}$, $\theta^{(p)} \in \Real^{2\times 16}$, $\theta^{(v)}\in\Real^{1\times 16}$ and $d=112$.
Given a state $s\in\State$, $\pi(\cdot|s;\theta)$ outputs the confidence in each action.
Each agent takes an action having $\arg \max_{a \in \action}\pi(\cdot|s;\theta)$ with probability $1-\alpha$.
With probability $\alpha$, the agent takes a randomly selected action.
$\alpha$ is decreased from $0.5$ to $0$ as $\alpha=\max\{0, 0.5-n/1800\}$.
Our implementation utilizes 9 threads for asynchronous agent processes.
Empirical codes are developed by Python 3.7.4, TensorFlow 1.14.0~\cite{tensorflow2015-whitepaper} and OpenAIGym 0.14.0~\cite{openAiGym}.

\textbf{Hyper Parameters.}
We set discounting factor $\gamma = 0.99$, learning rate $\eta=0.5$, loss scaling factors $\lambda=0.5$ and $\beta=0.01$.
$C$ is set $0.01$ for the Laplace mechanism and is set $1$ for the PRS.
For the PRS, we set $\hat{d}=\max\{1, \min\{d, \lfloor \varepsilon/2.5\rfloor\}\}$ as following \cite{Wang2019collecting}.

\textbf{Score.}
Each local agent measures a score, which is how long time steps the pole keeps the standings.
We observe how the learning progress reaches the target score ($\Theta=195$).
Especially to evaluate robustness across various environments, we measure the average score.
The average score at $n$-th submission over last $m$ (=10) submissions is:
\begin{equation}
    \mu_n = \frac{1}{m} \sum_{n' = n}^{n+m-1} \zeta_{n'}
    \label{eq:avgscore}
\end{equation}
where $\zeta_{n'}$ is the score at $n'$.

\begin{figure}[t!]
    \centering
    \subfloat[Laplace ($\varepsilon=1, |B|=1$)\label{fig:lap_eps1_bs1}]{
        \includegraphics[width=0.47\hsize]{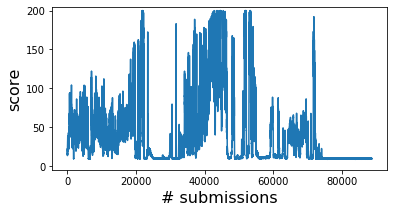}
    }
    \subfloat[Laplace ($\varepsilon=1, |B|=100$)\label{fig:lap_eps1_bs100}]{
        \includegraphics[width=0.47\hsize]{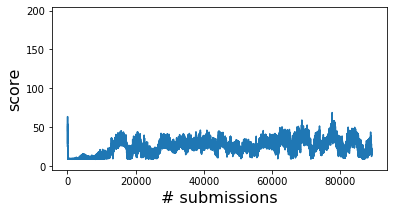}
    }\\
    \subfloat[PRS ($\varepsilon=1, |B|=1$)\label{fig:bf_eps1_bs1}]{
        \includegraphics[width=0.47\hsize]{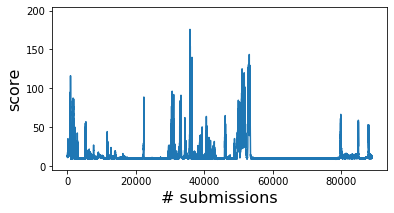}
    }
    \subfloat[PRS ($\varepsilon=1, |B|=100$)\label{fig:bf_eps1_bs100}]{
        \includegraphics[width=0.47\hsize]{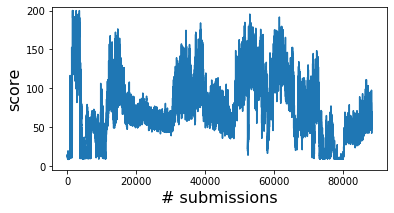}
    }
    \caption{The average scores during the training process.
    Buffering helps the PRS to train the policy, but disturbs the training of Laplace mechanism.}
    \label{fig:eval}
\end{figure}

\subsection{Observation of Learning Behaviors}

First, we observe the learning behaviors of our proposed methods to decide several hyper-parameters' values.
We compare two mechanisms with and without buffering, whose buffer size is $|B|=100$.
We regard training as a success if the average score meets $\Theta$.

Figure \ref{fig:eval} shows the average scores of our proposed method employing two different mechanisms during the training process with $\varepsilon=1$ and $|B|=\{1, 100\}$.
Without buffering, scores for all settings change drastically, but the buffering limits the learning dynamics of the Laplace mechanism too small to learn.
However, in the PRS mechanism, the buffering gives better learning stability than without buffering.
Thus, the buffering helps the PRS mechanism to train the policy well, but it disturbs the training with the Laplace mechanism.

In the later part of the evaluations, we employ the Laplace mechanism with $|B|=1$, and the PRS with $|B|=100$.

\subsection{Learning Efficiency}

We evaluate how early the algorithms achieve a target score.
To measure it, we define a metric \textit{first success time (FST)}:
\begin{equation}
    FST = \min \{n | n\in [N], \mu_n \geq \Theta \}.
    \label{eq:fst}
\end{equation}
We regard $FST$ as $\infty$ if a learning process cannot meet the success in 90,000 updates.

Table \ref{tbl:updates} shows the median of the $FST$ in $20$ trials for each setting.
The smaller median value suggests that a method is efficient to learn.
We measure the FST varying $\varepsilon=1, 2, 5$, and $10$.
The result of the Laplace mechanism shows a decreasing the median of $FST$ along with increasing $\varepsilon$.
While \method-A3C with PRS shows better results within $\varepsilon=\{2,5\}$, but it does not show such improvement at $\varepsilon=10$.
Therefore, \method-A3C with the PRS has a chance to increase learning efficiency than Laplace.

\begin{table}[t]
\centering
\caption{\#submissions that meets target score at first.}
\label{tbl:updates}
\begin{tabular}{lrrrrc}
    \toprule
    \multicolumn{1}{c}{} & \multicolumn{5}{c}{median of $FST$ (\ref{eq:fst})} \\
    $\mathcal{Q}$  & $\varepsilon=1$  & $\varepsilon=2$ & $\varepsilon=5$ & $\varepsilon=10$ & $\varepsilon=\infty$ \\
    \cmidrule(lr){1-1} \cmidrule(lr){2-6}
    Lap  & 18377.0 & 20238.5 & 5714.5 & 4055.0 &  \multirow{2}{*}{1769.0}\\ 
    PRS     & 25226.5 & 7549.0   & 2656.5 & 11217.5 &     \\
    \bottomrule
\end{tabular}
\end{table}

\begin{table}[t]
\centering
\caption{Success ratio.}
\label{tbl:success_ratio}
\begin{tabular}{lrrrrc}
    \toprule
    \multicolumn{1}{c}{} & \multicolumn{5}{c}{success ratio} \\
    $\mathcal{Q}$  & $\varepsilon=1$  & $\varepsilon=2$ & $\varepsilon=5$ & $\varepsilon=10$ & $\varepsilon=\infty$ \\
    \cmidrule(lr){1-1} \cmidrule(lr){2-6}
    Lap  & 0.80  & 0.90  & 1.00 & 1.00 & \multirow{2}{*}{1.0} \\ 
    PRS & 0.85  & 0.95  & 0.90 & 0.90 &   \\
    \bottomrule
\end{tabular}
\end{table}

\begin{figure}[t]
    \centering
    \subfloat[Laplace mechanism]{
        \includegraphics[width=.75\hsize]{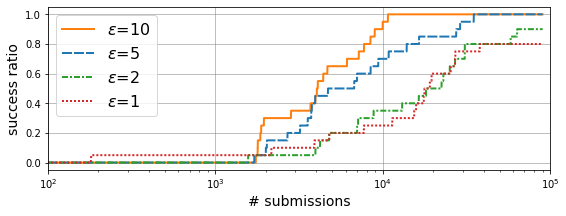}
        \label{fig:success_lap}
    }\\
    \subfloat[PRS mechanism]{
        \includegraphics[width=.75\hsize]{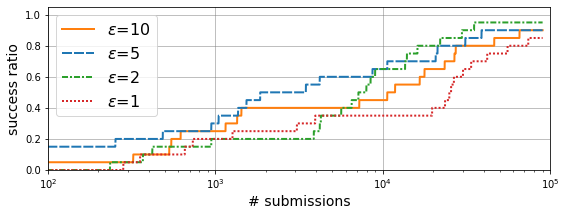}
        \label{fig:success_bf}
    }
    \caption{Success ratio for each $\varepsilon$ at \#submissions. \method-A3C with the PRS mechanism makes more successes in earlier stage than the method with the Laplace.}
    \label{fig:success}
\end{figure}


\begin{table}[t]
\centering
\caption{Relative AUC of Figure \ref{fig:success} against non-private.}
\label{tbl:auc}
\begin{tabular}{lrrrrc}
    \toprule
    \multicolumn{1}{c}{} & \multicolumn{5}{c}{relative area under curve} \\
    $\mathcal{Q}$  & $\varepsilon=1$  & $\varepsilon=2$ & $\varepsilon=5$ & $\varepsilon=10$ & $\varepsilon=\infty$ \\
    \cmidrule(lr){1-1} \cmidrule(lr){2-6}
    Lap & 0.673 & 0.711 & 0.909 & 0.965 &
    \multirow{2}{*}{1.0} \\
    PRS & 0.660 & 0.862 & 0.835 & 0.771 & \\
    \bottomrule
\end{tabular}
\end{table}

\subsection{Success Ratio}

We evaluate how many times proposed algorithms achieve the target score $\Theta$ over $K$ trials for each setting.
That is
\begin{equation}
    success\_ratio(n) = \frac{|\{k | FST_k \leq n\}_{k\in[K]}|}{K}
\end{equation}
where $FST_k$ is the $FST$ of $k$-th trial.
Here we set $K=20$.

Table \ref{tbl:success_ratio} shows the success ratios for various settings.
The non-private A3C ($\varepsilon=\infty$) succeeds in all trials.
With both of the randomized mechanisms, the algorithms tend to make more successes for larger $\varepsilon$.
The algorithm using PRS gives more successes for $\varepsilon=1, 2$, and the algorithm using the Laplace mechanism shows better for $\varepsilon=5, 10$.

Figure \ref{fig:success} plots the success ratio at $FST$.
The horizontal axis shows the number of submissions, and the vertical axis shows the ratio of the trials in which an average score exceeds $195$ by the update.
With larger $\varepsilon$, both algorithms achieve a high success ratio consuming fewer updates.
The algorithm using PRS makes more successes in the early stage, and the algorithm with Laplace shows more successes by the end of each training process.
This is due to larger loss of gradient by the PRS against the Laplace.

Table \ref{tbl:auc} shows the relative area under curve~(AUC) of Figure \ref{fig:success} against the AUC of non-private A3C.
An algorithm having a larger AUC is regarded as a better algorithm.
Laplace mechanism achieves larger AUC in proportional to $\varepsilon$.
While PRS shows the best at $\varepsilon=2$.

The PRS mechanism gives us more efficient DRL under LDP at some $\varepsilon$.
The PRS mechanism may be a better choice if we require strong privacy guarantees ($\varepsilon < 10$).
Otherwise, the Laplace mechanism seems more promising.

\section{Conclusion}

We studied locally differentially private algorithms for distributed reinforcement learning to obtain a robust policy that performs well across distributed private environments.
We proposed a general framework \method, and its concrete algorithm  \method-A3C with two randomized mechanism for injecting randomness.
Our proposed algorithm leans a robust policy based on the reported noisy gradients that satisfy LDP from local agents.
Without any privacy concerns of the local agents, the algorithm can update a global model to make it robust across various environments.
We also demonstrated how our method learns the robust policy effectively even it is required to satisfy local differential privacy.
This work enables us to obtain a robust agent that performs well across distributed private environments.

\clearpage

\bibliography{ref}
\bibliographystyle{abbrv}

\end{document}